\documentclass[12pt,final]{l4dc2020} 

\title{Localized active learning of Gaussian process state space models}
\usepackage{times}

\usepackage{float}
\usepackage{amsmath,amssymb,amsfonts,bm}
\usepackage{subcaption}
\usepackage{algorithm}
\usepackage{algorithmic}
\usepackage{graphicx}
\usepackage{textcomp}
\usepackage[utf8]{inputenc}
\usepackage{cleveref}
\usepackage{graphicx}
\usepackage{epstopdf}
\usepackage{enumerate}
\usepackage{array}
\usepackage{makecell}
\usepackage{cellspace}
\usepackage{braket}
\usepackage{mathtools}
\usepackage{siunitx}
\usepackage{blindtext}
\usepackage{etoolbox}
\usepackage{algorithmic,algorithm}
\usepackage{epstopdf}
\newtheorem{assumption}{Assumption}
\newtheorem*{*proofofthm1}{Proof of Theorem 1}

\Crefname{algorithm}{Algorithm}{Algorithms}

\newcommand\xaug{{\tilde{\bm x}}}

\newcommand{\x}{{\bm x}}

\newcommand{\StatSpAug}{{\tilde{\StatSp}}}
\newcommand{\y}{\bm y}
\newcommand{\ysc}{y}

\newcommand{\xsc}{x}
\newcommand{\gsc}{g}

\newcommand{\X}{{\mathcal{X}}}
\newcommand{\StatSpAugRef}{{\StatSpAug_{\text{ref}}}} 
\newcommand{\xref}{\bm{\xi}}
\newcommand{\Xref}{{\Xaug_{\text{ref}}}}

\newcommand{\K}{{\bm{K}}}
\newcommand{\f}{{\bm f}}
\newcommand{\fsc}{f}
\newcommand{\g}{{\bm{g}}}
\newcommand{\gpred}{\widehat{\bm{g}}}

\newcommand{\w}{{\bm{w}}} 
\newcommand{\wsc}{{w}} 

\newcommand{\transp}{^{\text{T}}}

\newcommand{\Id}{\bm I}
\newcommand{\NMPC}{N_H}
\newcommand{\dimx}{{d_x}} 
\newcommand{\dimu}{{d_u}}

\newcommand{\uin}{{\bm u}}
\newcommand{\xaugsc}{\tilde{x}}
\newcommand{\StatSp}{\mathcal{X}} 

\newcommand{\kernel}{k}
\newcommand{\mean}{m}
\newcommand{\Xaug}{{\tilde{\bm{X}}}}

\newcommand{\postmeansc}{\mu}
\newcommand{\step}{t}
\newcommand{\MPCstep}{\tau} 
\newcommand{\sigon}{\sigma_{w}}
\newcommand{\Sigmaon}{\bm{\Sigma}_{w}}
\newcommand{\InputSp}{{\mathcal{U}}}
\newcommand{\rationals}{\mathbb{R}}

\newcommand{\rationalspz}{\rationals_{+,0}}

\newcommand{\naturalsz}{\mathbb{N}_{0}}

\DeclarePairedDelimiterX\PBasics[1](){ #1}
\DeclarePairedDelimiterX\EBasics[1][]{ #1}

\DeclarePairedDelimiter\absv{\lvert}{\rvert}

\DeclareMathOperator*{\argmax}{arg\,max}

\author{%
 \Name{Alexandre Capone} \Email{alexandre.capone@tum.de}\\
 \Name{Gerrit Noske} \Email{ga24coq@mytum.de}\\
 \Name{Jonas Umlauft} \Email{jonas.umlauft@tum.de}\\
  \Name{Thomas Beckers} \Email{t.beckers@tum.de}\\
   \Name{Armin Lederer} \Email{armin.lederer@tum.de}\\
 \Name{Sandra Hirche} \Email{hirche@tum.de}\\
 \addr Chair of Information-oriented Control \\
 Department of Electrical and Computer Engineering\\
 Technical University of Munich \\
 D-80333 Munich, Germany
}

\begin{document}
	
	\setlength{\abovedisplayskip}{8pt}
	\setlength{\belowdisplayskip}{8pt}
	\setlength{\abovedisplayshortskip}{8pt}
	\setlength{\belowdisplayshortskip}{8pt}
	\setlength{\textfloatsep}{8pt}

\maketitle

\begin{abstract}%
While most dynamic system exploration techniques aim to 
achieve a globally accurate model, this is generally unsuited for 
systems with unbounded state spaces. Furthermore, many applications do not require a globally accurate model, 
e.g., local stabilization tasks. In this paper, we propose an active learning 
strategy for Gaussian process state space models that aims to obtain an 
accurate model on a bounded subset of the state-action space. Our approach aims 
to maximize the mutual information of the exploration trajectories with respect 
to a discretization of the region of interest. By employing model predictive 
control, the proposed technique integrates information collected during 
exploration and adaptively improves its exploration strategy. To enable 
computational tractability, we decouple the choice of most informative data 
points from the model predictive control optimization step. This yields two 
optimization problems that can be solved in parallel. We apply the proposed 
method to explore the state space of various dynamical systems and compare our 
approach to a commonly used entropy-based exploration strategy. In all 
experiments, our method yields a better model within the region of interest 
than the entropy-based method.

\end{abstract}

\section{Introduction}
\label{sect:Introduction}

Autonomous systems often need to operate in complex environments, of which a model is difficult or even impossible to derive from first principles. 
Learning-based techniques have become a promising paradigm to address these issues \citep{pillonetto2014kernel}. 
In particular, Gaussian processes (GPs) have been increasingly employed for 
system identification and control \citep{Umlauft2018, Capone2019BacksteppingFP, 
berkenkamp2015safe,deisenroth2015gaussian}. GPs possess very good generalization 
properties \citep{Rasmussen2006}, which can be leveraged to obtain 
data-efficient learning-based approaches \citep{deisenroth2015gaussian, 
kamthe2017data}. By employing a Bayesian framework, GPs provide an automatic 
trade-off between model smoothness and data fitness. Moreover, GPs provide an 
explicit estimate of the model uncertainty that is used to derive probabilistic 
bounds in control settings \citep{Capone2019BacksteppingFP,beckers2019automatica,umlauft2020feedback}.

A performance-determining factor of data-driven techniques is the 
quality of the available data. In settings where data is insufficient to 
achieve accurate predictions, new data needs to be gathered via exploration 
\citep{umlauft2020feedback}. 
In classical reinforcement learning settings, exploration is often enforced by randomly selecting a control action with a predetermined probability that tends to zero over time \citep{dayan1996exploration}. However, this is generally inefficient, as regions of low uncertainty are potentially revisited in multiple iterations. 
These issues have been addressed by techniques that choose the most informative exploration trajectories \citep{alpcan2015information,ay2008predictive,burgard2005probabilistic,schreiter2015safe}. The goal of these methods is to obtain a globally accurate model. While this is reasonable for systems with a bounded state-action space, it is unsuited for systems with unbounded ones, particularly if a non-parametric model is used. This is because a potentially infinite number of points is required to achieve a globally accurate model. Furthermore, in practice a model often only needs to be accurate locally, e.g., for stabilization tasks.

In this paper, we propose a model predictive control-based exploration approach that steers the system towards the most informative points within a bounded subset of the state-action space. By modeling the system with a Gaussian process, we are able to quantify the information inherent in each data point. Our approach chooses actions by approximating the mutual information of the system trajectory with respect to a discretization of the region of interest. This is achieved by first selecting the single most informative data point within the region of interest, then steering the system towards that point using model predictive control. Through this approximation, the solution approach is rendered computationally tractable.

The remainder of this paper is structured as follows. \Cref{sect:ProblemSetting} describes the general problem. \Cref{sect:GaussianProcesses} discusses how GPs are employed for modeling and exploration. \Cref{sect:ExploratoryMPC} presents the LocAL algorithm, and is followed by a numerical simulation example, in~\Cref{sect:Examples}.

\section{Problem Statement}
\label{sect:ProblemSetting}
We consider the problem of exploring the state and control space of a 
discrete-time nonlinear system with Markovian dynamics of the 
form
%
\begin{align}
	\label{eq:SystemDynamics}	
	{\x_{\step+1}} &= \f (\x_{\step},\uin_{\step}) + \g(\x_{\step},\uin_{\step}) + \w_{\step} := \f (\xaug_{\step}) + \g(\xaug_{\step}) + \w_{\step},
\end{align}
where~$\step\in\naturalsz$,~${\x_{\step} \in \StatSp \subseteq 
	\rationals^{\dimx}}$ and 
${\uin_{\step} \in \mathcal{U} \subseteq \rationals^{\dimu}}$ are the 
system's state vector and control vector at the~$\step$-th time step, 
respectively. The system is disturbed by multivariate Gaussian process noise~${\w_{\step} \sim 
\mathcal{N}(\bm{0}, \Sigmaon)}$ with~$\Sigmaon = \text{diag}(\sigma_{\text{w},1}^2, 
\ldots, 
\sigma_{\text{w},\dimx}^2)$,~$\sigma_{\text{w,i}}^2\in\rationalspz$. The 
concatenation~${\xaug_{\step}:= 
	\begin{pmatrix}
	\x_{\step}\transp& \uin_{\step}\transp\end{pmatrix}\transp \in 
	\StatSpAug}$, where~${\StatSpAug:=\StatSp \times \mathcal{U}}$
 is employed for simplicity of exposition. 
The nonlinear function~${\f : \StatSpAug 
	\rightarrow \StatSp}$ 
represents the known component of the system dynamics, e.g., a model obtained using first 
principles, while~${\g : \StatSpAug \rightarrow \StatSp}$ 
corresponds to the unknown component of the system dynamics. 

We aim to obtain an approximation of the function 
$\g(\cdot)$, denoted~$\gpred(\cdot)$, which provides an accurate estimate of~${\g(\cdot)}$ on a predefined bounded subset of the augmented state space $\StatSpAug_B \subset \StatSpAug$. This is often required in practice, e.g., for local stabilization tasks.

\section{Gaussian Processes}
\label{sect:GaussianProcesses}
In order to faithfully capture the stochastic behavior of 
\eqref{eq:SystemDynamics}, we model the system as a \textit{Gaussian process} 
(GP), where we employ measurements of the augmented state vector 
$\xaug_{\step}$ as training inputs, and the differences~$\x_{\step+1} - 
\f(\xaug_{\step})= \g(\xaug_{\step}) + \w_{\step}$ as training targets. 

A GP is a collection of dependent random variables variables, for which any 
finite subset is jointly normally distributed \citep{Rasmussen2006}. It is 
specified by a mean function~$\mean: \StatSpAug \rightarrow \mathbb{R}$ and a positive definite covariance 
function~$\kernel: \StatSpAug \times \StatSpAug \rightarrow \mathbb{R}$, also known as kernel. In this paper, we 
set~$\mean \equiv 0$ without loss of generality, as all prior knowledge is already encoded in $\fsc(\cdot)$. The kernel~$\kernel(\cdot,\cdot)$ is 
a similarity measure for evaluations of~$\g(\cdot)$, and encodes function 
properties such as smoothness and periodicity. 

In the case where the state is a scalar, i.e.,~$\dimx=1$, given~$n$ training 
input samples~\linebreak$\Xaug = \left\{\xaug_1, \ldots, \xaug_n \right\}\subset 
\StatSpAug$ and training 
outputs~$\y_{\Xaug} = \begin{pmatrix}\gsc(\xaug_1) + \wsc_1& 
\ldots&\gsc(\xaug_n) + \wsc_n \end{pmatrix} \transp$, the posterior mean and 
variance of the GP corresponds 
to a one-step transition model. Starting at a point~$\xaug_{\step}$, the 
difference between the subsequent state and the known component is normally 
distributed, i.e.,
\begin{align}
	\x_{\step+1} - \f(\xaug_{\step})\sim 
	\mathcal{N}\left(\postmeansc_n(\xaug_{\step}) ,\sigma_n^2(\xaug_{\step})  \right),
\end{align}
with mean and variance given by
\begin{align*}
	\postmeansc_n(\xaug_{\step}) \coloneqq 
	\bm{\kernel}^{\text{T}}(\xaug_{\step}) \left(\K + \sigon^2 \Id 
	\right)^{-1}\y_{\Xaug}, \qquad
	\sigma^2_n(\xaug_{\step}) 
	\coloneqq 
	\kernel(\xaug_{\step},\xaug_{\step}) - 
	\bm{\kernel}\transp(\xaug_\step)\left(\K + \sigma^2 \Id \right)^{-1} 
	\bm{\kernel}(\xaug_{\step}),
\end{align*}
respectively, where~${\bm{\kernel}(\cdot) = \begin{pmatrix}
	\kernel(\xaug_1, \cdot) &\ldots& k(\xaug_n, \cdot)\end{pmatrix}\transp}$, 
and the entries of the covariance matrix~$\bm{K}$ are computed as $K_{ij}=\kernel(\xaug_i,\xaug_j)$, $i,j=1,\ldots,n$. 

In the case where the state is multidimensional, we model dimension of the state transition function
using a separate GP. This corresponds to the assumption that the state transition function entries 
are conditionally independent. For simplicity of exposition, unless stated 
otherwise, we henceforth assume~$\dimx =1$. However, the methods presented in 
this paper extend straightforwardly to the multivariate case.
\subsection{Performing multi-step ahead predictions}
The GP model presented in the previous section serves as a one-step predictor 
given a known test input~$\xaug_\step$. However, if only a 
distribution~$p(\xaug_{\step})$ is available, the successor states' distribution generally cannot be computed 
analytically. Hence, the distributions of future states cannot be computed 
exactly, but only approximated, e.g., using Monte Carlo methods
\citep{candela2003multistep}. Alternatively, approximate computations exist that enable to propagate the GP uncertainty over multiple time steps, such as 
moment-matching and GP linearization \citep{deisenroth2015gaussian}. In this 
paper, we employ the GP mean to perform multi-step ahead predictions, without 
propagating uncertainty, i.e., $
	\xsc_{\step+1} = \fsc(\xaug_{\step}) + \mu_n(\xaug_{\step})$, $\step\in \mathbb{N}$.
However, the proposed method is also applicable using models that propagate uncertainty, e.g., moment-matching \citep{deisenroth2015gaussian}.

\subsection{Quantifying utility of data}
\label{sect:QuantifyingUtilityofData}

In order to steer the system along informative trajectories, we need to quantify the utility of data points in the augmented state space $\StatSpAug$. To this end, we consider the 
\text{mutual information} between observations~$\y_{\Xaug}$
at training inputs~$\Xaug$ and evaluations $\y_{\Xref}$ at reference points $\Xref$. Here~${\Xref\subset \StatSpAugRef}$ is a discretization of the bounded subset $\StatSpAugRef$.
Formally, the mutual information between~$\y_{\Xref}$ and~$\y_{\Xaug}$ is given by
\begin{align}
\label{eq:MutualInformation}
I(\y_{\Xaug}, \y_{\Xref}) 
= \!\!\!\!\!\!\! \int \limits_{\StatSp^{\lvert\Xref\rvert } \times \StatSp^{\lvert\Xaug\rvert }}\!\!\!\!\!\!\! p\left(\y_\Xaug, \y_\Xref \right) \log\left(\frac{p\left(\y_\Xaug, \y_\Xref \right)}{p\left(\y_\Xaug\right),p\left( \y_\Xref\right)}\right) d \y_\Xaug d \y_\Xref
\end{align}
respectively denote the \text{differential entropy} of~$\y_\Xaug$ and 
the \text{conditional differential entropy} of~$\y_\Xaug$ 
given~$\y_\Xref$. In practice, computing \eqref{eq:MutualInformation} for a multi-step GP prediction is intractable. However, we can obtain the single most informative data point $\xref^* \in \StatSpAug$  with respect to~$\y_{\Xref}$ by computing the unconstrained minimum of
\begin{align}
\label{eq:GP_MutualInformation}
	I(\y_{\xref}, \y_{\Xref}) = \frac{1}{2} \log \left(\frac{{\left(\kernel(\xref,\xref)+\sigon\right)}\lvert\bm{K}_{\Xref} +\sigon \Id \rvert }{\lvert{\bm{K}_{ \Xref}\cup \xref} +\sigon \Id  \rvert}\right),
\end{align}
where $\lvert \cdot \rvert$ denotes the determinant of a square matrix. In settings with unconstrained decision spaces, sequentially computing a minimizer of \eqref{eq:GP_MutualInformation} has been shown to yield a solution that corresponds to at least $63\%$ of the optimal value \citep{krause2008near}.

\section{The LocAL algorithm}
\label{sect:ExploratoryMPC}

\begin{algorithm}[t]
	\caption{LocAL (Localized Active Learning)}
	
	\begin{algorithmic}[1]
		\label{alg:EMPC}
		\REQUIRE{$\x_0$,~$\f(\cdot)$, $\kernel(\cdot,\cdot)$}
		\FOR{$\step = 0, 1, 2, 3, \ldots$}
		\STATE{Solve~\vspace*{-0.4cm}\begin{align*}
			&\xref^{*} = \arg \max \limits_{ \xref \in \StatSpAug} \  \frac{1}{2} \log \left(\frac{{\left(\kernel(\xref,\xref)+\sigon\right)}\lvert\bm{K}_{\Xref} +\sigon \Id \rvert }{\lvert{\bm{K}_{ \Xref}\cup \xref} +\sigon \Id  \rvert}\right),
			\end{align*}\vspace*{-0.7cm}}
		\STATE{Solve~\vspace*{-0.4cm}\begin{align*}
			&\bm{U}^{*} = \arg \min \limits_{\bm{U} \in \mathcal{U}^{\NMPC} } 	\sum \limits_{\step =1}^{\NMPC} 
			\left( \xref^* - \xaug_\step\right)\transp \bm{Q} \left( \xref^* - \xaug_\step\right) \\
			\text{s.t.}  \quad &
			\xsc_{\step+\MPCstep+1} = 	\fsc(\xaug_{\step+\MPCstep}) + \mu_{\step}(\xaug_{\step+\MPCstep}), \ 
			\uin_{\step+\MPCstep} \in \InputSp, \quad 
			\forall \MPCstep \in \left\{0, \ldots, \NMPC-1\right\}
			\end{align*}\vspace*{-0.7cm}}
		\STATE{Apply~$\uin_{\step+1}^*$ to system}
		\STATE{Measure~$\x_{\step+1}$, set $\Xaug = \Xaug \cup \x_{\step+1}$, and update GP 
			model~$\mu_{\step}(\cdot)$,~$\sigma^2_\step(\cdot)$}
		\ENDFOR
	\end{algorithmic}
\end{algorithm}

The system dynamics \eqref{eq:SystemDynamics} considerably limit the decision space at every time step $\step$. Furthermore, after a data point is collected, both the GP model and mutual information change. Hence, we employ a model predictive control (MPC)-based approach to steer the system towards areas 
of high information. 
Ideally, at every MPC-step $\step$, we would like to minimize \eqref{eq:MutualInformation} with respect to a series of $\NMPC$ inputs $\bm{U} \coloneqq \{ \uin_{\step},\ldots, \uin_{\step+\NMPC-1} \}$. However, this is generally infeasible, limiting its applicability in an MPC setting. 
Hence, we consider an approximate solution approach that sequentially computes the most informative data point by minimizing \eqref{eq:GP_MutualInformation} separately from the MPC optimization. This is achieved as follows. At every time step $\step$, an unconstrained minimizer $\xref^*$ of \eqref{eq:GP_MutualInformation} is computed. Afterwards, the MPC computes the approximate optimal inputs $\bm{U}^*$ by minimizing a constrained optimization problem that penalizes the weighted distance to the reference point $\xref$ using a positive semi-definite weight matrix $\bm{Q}$.
The ensuing state is then measured, the GP model is updated, and the procedure is repeated. These steps yield the Localized Active Learning (LocAL) algorithm, which is presented in \Cref{alg:EMPC}.

The square weight matrix $\bm{Q} \in \mathbb{R}^{\dimx+\dimu} \times \mathbb{R}^{\dimx+\dimu}$ should be chosen such that the MPC steers the system as close to $\xref$ as possible. This represents a system-dependent task. Alternatively, $\bm{Q}$ can be chosen such that the MPC cost function corresponds to a quadratic approximation of the mutual information, e.g., such that $
	I(\y_{\xaug_\step}, \y_{\xref}) \approx \sum_{\step =1}^{\NMPC} 
	\left( \xref - \xaug_\step\right)\transp \bm{Q} \left( \xref - \xaug_\step\right)$
holds for $\xaug_\step \approx \xref$.



\subsection{Sensitivity analysis}
\label{sensitivity_analysis}

We now provide a sensitivity analysis of \eqref{eq:GP_MutualInformation}. Specifically, we show that the difference in mutual information at $\xref^*$ and $\xaug_{\step}$ is lower bounded. To this end, we require the following assumption.
\begin{assumption}
	\label{assumption:kernelassumptions} 
	The kernel $k(\cdot,\cdot)$ is Lipschitz continuous and upper bounded by $k_{\text{max}}>0$.
\end{assumption}
This assumption holds for many commonly used kernels, e.g., the squared exponential kernel.

Using \Cref{assumption:kernelassumptions}, we are able to bound the difference in mutual information at $\bm{\zeta}^*$ and an arbitrary state $\xaug_{\step}$, as detailed in the following.

\begin{theorem}
	\label{theorem:MainResult}
	Let~$\Delta \xref^*_{\step} := \xref^*- \xaug_{\step}$ be the difference between the augmented state and the most informative data point $\xref^*$  at time step $\step$. 
	Moreover, let~$\Delta I^*_{\step} 
	\coloneqq I(\y_{\xref^*}, \y_{\Xref}) - I(\y_{\xaug_{\step}}, \y_{\Xref}) $ denote the 
	corresponding difference in mutual information, and let \Cref{assumption:kernelassumptions} hold. Then, there exists a constant $L \geq 0$, such that $
	 \Delta I^*_t 
	 \leq \lvert \Xref \rvert \log \left(1+ {C(\Delta 
		\bm{x}^*_{\step})}\right) /2$
	holds, where $C(\Delta \bm{x}^*_{\step}) \coloneqq {\sigon^{-2}} \min \{ k_{\text{max}}, 
	L({t+\lvert \Xref\rvert})^{-1/2} \lVert \Delta \bm{x}^*_{\step} 
	\rVert_2\}$.
\end{theorem}
To prove \Cref{theorem:MainResult}, we require the following preliminary results. 

\begin{lemma}
	\label{lemma:InformationGaintoSummationoverLogs}
	For any $i \in \left\{1,\ldots,\lvert \Xref\rvert \right\}$, let $\tilde{\bm{X}}_{\text{ref},i}:=\left\{\xaug_{\text{ref},1},\ldots,\xaug_{\text{ref},i}\right\} \subseteq \Xref$ denote the subset containing the first $i$ elements of $\Xref$. Then
	\begin{align}\begin{split}
	&\argmax_{\Xaug\subset\StatSpAug,\absv{\Xaug}=n} 	I(\y_{\mathcal{X}}, 
	\y_{\mathcal{A}})  = \argmax_{\Xaug\subset\StatSpAug,\absv{\Xaug}=n}  -\sum 
	\limits_{i=0}^{\lvert \Xref \rvert-1}\log(2\pi 
	e(\sigon^2 + \sigma^2_{i+\lvert \Xaug \rvert 
	}(\xaug_{\text{ref},i+1})))/2,
	\end{split}
	\end{align}
	where~$\sigma_{i+\lvert \Xaug \rvert }(\xaug_{\step}) \coloneqq  
	\sigma(\xaug_{\text{ref},\step} \vert \y_{\Xaug \cup \tilde{\bm{X}}_{\text{ref},i}})$ denotes the posterior GP variance evaluated 
	at~$\xaug_{\text{ref},\step}$ conditioned on 
	training data observed at~$\Xaug \cup \tilde{\bm{X}}_{\text{ref},i}$.
\end{lemma}
\begin{proof}
	The proof follows directly from the proof of \cite[Lemma 5.3]{Srinivas2012}. 
	Let~$\ysc_i \in \y_{\Xaug}$ denote a single training target. For a GP, 
	\begin{align}\begin{split}	
	H(\y_{\Xaug}) &= H(\y_{\Xaug} \backslash \ysc_n) + H(\ysc_n  \vert 
	\y_{\Xaug} \backslash \ysc_n) = H(\y_{\Xaug} \backslash \ysc_n) + 
	\log\left(2\pi e(\sigon^2 + \sigma^2_{n-1}(\xaug_n))\right)/2 \\&= 
	\sum \limits_{i=1}^{n-1}\log\left(2\pi e(\sigon^2 + 
	\sigma^2_{i}(\xaug_{i+1}))\right)/2
	\end{split}
	\end{align}
	holds. Hence,
	\begin{align}
	\begin{split}
	H(\y_{\Xaug} \cup \y_{\mathcal{A}}) =H(\y_{\Xaug}) + \sum 
	\limits_{i=1}^{\lvert \Xref \rvert-1}\log\left(2\pi e(\sigon^2 + 
	\sigma^2_{i+\lvert \Xaug \rvert }(\xaug_{\text{ref},i+1}))\right)/2
	\end{split}
	\end{align}
	Since, the points~$\y_{\mathcal{A}}$ are fixed, optimizing 
	\eqref{eq:MutualInformation} with respect to~$\y_{\Xaug}$ is equivalent to 
	optimizing~$H(\y_\Xaug) - H(\y_{\Xref} \cup \y_{\Xaug})= -\sum_{i=1}^{\lvert 
		\Xref \rvert-1}\log(2\pi e(\sigon^2 + \sigma^2_{i+\lvert \Xaug 
		\rvert 
	}(\xaug_{\text{ref},i+1})))/2$. 
\end{proof}
Hence, we employ the objective function
\begin{align}
\label{eq:ObjectiveFunction}
\begin{split}
J(\Xaug) &= H(\y_{\Xaug}) - H(\y_{\Xaug} \cup \y_{\Xref}) = \sum \limits_{i=1}^{\lvert \Xref \rvert-1} -\log(2\pi e(\sigon^2 + \sigma^2_{i+\lvert \Xaug 
	\rvert }(\xaug_{\text{ref},i+1})))/2.
\end{split}
\end{align}
with 
\begin{align}
J_{i}(\Xaug) = -\log\left(2\pi e(\sigon^2 + \sigma^2_{i+\lvert \Xaug 
	\rvert }(\xaug_{\text{ref},i+1}))\right)/2.
\end{align}
\begin{proof}
	This follows directly from the property $
	I(\y_{\mathcal{X}}, \y_{\mathcal{A}}) \leq I(\y_{\mathcal{A}}, 
	\y_{\mathcal{A}})$ and Lemma \ref{lemma:InformationGaintoSummationoverLogs}.
\end{proof}

\begin{lemma}
	Let $\K_{\step}$ be the covariance matrix used to compute $\sigma^2_n(\cdot)$, and let $\partial \K_{\step}  / \partial \xaugsc_{ij}$ denote the matrix of partial derivatives of $\bm{K}_{\step}$ with respect to the $i$-th entry of the $j$-th data point $\xaug_j$. Then $
	\lambda_{\text{max}} \left( \partial \K_{\step} / \partial \xaugsc_{ij} \right) \leq 2L_k\sqrt{t}$ holds, where $\lambda_{\text{max}}(\cdot)$ denotes the maximal eigenvalue.
\end{lemma}
\begin{proof}
	The matrix of partial derivatives has a single nonzero column and a single nonzero row. Its entries are
	\begin{align}
	\left[ \frac{\partial \K_{\step}  }{\partial \xaugsc_{ij}} \right]_{\alpha \beta} = \begin{cases}
	\frac{\partial \kernel(\xaug_\alpha, \xaug_\beta)  }{\partial \xaugsc_{i\alpha}}, & \text{if} \quad \alpha = j \\
	\frac{\partial \kernel(\xaug_\alpha, \xaug_\beta)  }{\partial \xaugsc_{i\beta}}, & \text{if} \quad \beta = j \\
	0 & \text{otherwise}
	\end{cases}
	\end{align}
	and its eigenvalues are given by
	\begin{align*}
	\lambda \left(\frac{\partial \K_{\step}  }{\partial \xaugsc_{ij}} \right) = 
	\frac{\frac{\partial \kernel(\xaug_j, \xaug_j)  }{\partial 
			\xaugsc_{ij}} \!\pm\! \sqrt{\left(\frac{\partial \kernel(\xaug_j, \xaug_j)  
			}{\partial \xaugsc_{ij}}\right)^2 \!+\! 4 \sum\limits_{\alpha=1,\alpha \neq 
				j}^\step \left(\frac{\partial \kernel(\xaug_j, \xaug_\alpha)  
			}{\partial \xaugsc_{ij}} \right)^2 }}{2} \leq \frac{L_k 
		\left(1\!+\!\sqrt{1\!+\!4\step} \right)}{2} \leq 2L_k\sqrt{t}.
	\end{align*}
\end{proof}
\begin{corollary}
	Let \Cref{assumption:kernelassumptions} hold. Then for a fixed set of $n$ data points the posterior covariance $\sigma^2_n(\xaug_{\text{ref}})$  is Lipschitz continuous with respect to the data, with Lipschitz constant given by $L_{\sigma_n} = 2L_{\kernel} \sqrt{\kernel_{\text{max}}}/\sigma_{\text{on}}(1 +\sqrt{n \kernel_{\text{max}}}/\sigma_{\text{on}} )$.
\end{corollary}
\begin{proof}	
	For an arbitrary $\xaug_{\text{ref}} \in \Xref$, data point $\xaug_j$, and corresponding entry $\xaugsc_{ij}$,
	\begin{align*}
	\begin{split}
	\frac{\partial \sigma^2_n(\xaug_{\text{ref}})}{\partial \xaugsc_{ij}}  = &-2 
	\frac{\partial \bm{\kernel}(\xaug_{\text{ref}})}{\partial \xaugsc_{ij}}\transp \left(\K + \sigma_{\text{on}}^2 \Id \right)^{-1} 
	\bm{\kernel}(\xaug_{\text{ref}}) \\ &-\bm{\kernel}(\xaug_{\text{ref}})\transp \left(\K + \sigma_{\text{on}}^2 \Id \right)^{-1} \frac{\partial \K  }{\partial \xaugsc_{ij}} \left(\K + \sigma_{\text{on}}^2 \Id \right)^{-1}
	\bm{\kernel}(\xaug_{\text{ref}}) \\
	\leq & 2 \lambda_{\text{max}} \left( \left(\K + \sigma_{\text{on}}^2 \Id \right)^{-\frac{1}{2}} \right) \Big\rVert\frac{\partial \bm{\kernel}(\xaug_{\text{ref}})}{\partial \xaugsc_{ij}} \Big\lVert_2\Big\lVert \left(\K + \sigma_{\text{on}}^2 \Id \right)^{-\frac{1}{2}}
	\bm{\kernel}(\xaug_{\text{ref}}) \Big\rVert_2  \\
	& + \lambda^2_{\text{max}} \left(  \left(\frac{\partial \K  }{\partial \xaugsc_{ij}}\right)^{\frac{1}{2}} \left(\K + \sigma_{\text{on}}^2 \Id \right)^{-\frac{1}{2}} \right) \Big\lVert \left(\K + \sigma_{\text{on}}^2 \Id \right)^{-\frac{1}{2}}
	\bm{\kernel}(\xaug_{\text{ref}}) \Big\rVert_2^2 \\
	\leq &  \frac{2}{\sigma_{\text{on}}}  L_{\kernel} \Big\lVert \left(\K + \sigma_{\text{on}}^2 \Id \right)^{-\frac{1}{2}}
	\bm{\kernel}(\xaug_{\text{ref}}) \Big\rVert_2  +  \frac{2L_{\kernel}\sqrt{n}}{\sigma_{\text{on}}^2} \Big\lVert \left(\K + \sigma_{\text{on}}^2 \Id \right)^{-\frac{1}{2}}
	\bm{\kernel}(\xaug_{\text{ref}}) \Big\rVert_2^2 \\ 
	\leq & \frac{2}{\sigma_{\text{on}}}  L_{\kernel} \sqrt{\kernel_{\text{max}}}  +  \frac{2L_{\kernel}\sqrt{n}}{\sigma_{\text{on}}^2} \kernel_{\text{max}}.
	\end{split}
	\end{align*}
	Here we employ the identities
	\begin{align}
	\bm{\kernel}(\xaug_{\text{ref}})\transp \frac{\partial \left(\K + \sigma_{\text{on}}^2 \Id \right)^{-1}  }{\partial \xaugsc_{ij}} 	\bm{\kernel}(\xaug_{\text{ref}}) &= \bm{\kernel}(\xaug_{\text{ref}})\transp \left(\K + \sigma_{\text{on}}^2 \Id \right)^{-1} \frac{\partial \K  }{\partial \xaugsc_{ij}} \left(\K + \sigma_{\text{on}}^2 \Id \right)^{-1}
	\bm{\kernel}(\xaug_{\text{ref}})
	\end{align}
	and
	\begin{align}
	\Big\lVert \left(\K + \sigma_{\text{on}}^2 \Id \right)^{-\frac{1}{2}}
	\bm{\kernel}(\xaug_{\text{ref}}) \Big\rVert_2^2 &= \bm{\kernel}(\xaug_{\text{ref}})\transp \left(\K + \sigma_{\text{on}}^2 \Id \right)^{-1} \bm{\kernel}(\xaug_{\text{ref}}) \leq \kernel(\xaug_{\text{ref}},\xaug_{\text{ref}}) \leq \kernel_{\text{max}},
	\end{align}
	where the latter identity follows from the positivity of $\sigma^2_n(\cdot)$.
\end{proof}
We can now prove \Cref{theorem:MainResult}
\begin{*proofofthm1}
	Define $\sigma^2_{i+\step,*}\!(\cdot) \!\!\coloneqq\!\! \sigma^2(\cdot \vert \y_{\Xaug_\step \cup \tilde{\bm{X}}_{\text{ref},i}\cup \x^*_{\step}} \!)$ and $\sigma^2_{i+\step}(\cdot) \!\!\coloneqq\!\! \sigma^2(\cdot \vert \y_{\Xaug_\step \cup \tilde{\bm{X}}_{\text{ref},i}\cup \x_{\step}} \!)$. Moreover, assume without loss of generality that $\sigma^2_{i+\step,*}\!(\xaug_{\text{ref},i+1}\!) \!\!\geq\!\! \sigma^2_{i+\step}(\xaug_{\text{ref},i+1}\!)$ holds. Then, due to Assumption \ref{assumption:kernelassumptions} and Lemma \ref{lemma:InformationGaintoSummationoverLogs},
	\begin{align}
	\label{eq:InformationLoss}
	\begin{split}
	\Delta I^*_{\step} &=   -\sum 
	\limits_{i=0}^{\lvert \Xref \rvert-1}\log\left(2\pi 
	e(\sigon^2 + \sigma^2_{i+\step,*}(\xaug_{\text{ref},i+1}))\right)/2 + \sum 
	\limits_{i=0}^{\lvert \Xref \rvert-1}\log(2\pi 
	e(\sigon^2 + \sigma^2_{i+\step}(\xaug_{\text{ref},i+1})))/2 \\
	&= \underbrace{\sum 
		\limits_{i=0}^{\lvert \Xref \rvert-1}\frac{1}{2}\log\left(\frac{\sigon^2 + \sigma^2_{i+\step}(\xaug_{\text{ref},i+1})}{\sigon^2 + \sigma^2_{i+\step,*}(\xaug_{\text{ref},i+1})}\right)}_{\leq \lvert \Xref \rvert \frac{1}{2}\log \left(1+ \frac{k_{\text{max}}}{\sigma_{\text{on}}^2}\right)}\\&= \underbrace{\sum 
		\limits_{i=0}^{\lvert \Xref \rvert-1}\frac{1}{2}\log \left(1 + \frac{\sigma^2_{i+\step}(\xaug_{\text{ref},i+1})-\sigma^2_{i+\step,*}(\xaug_{\text{ref},i+1})}{\sigon^2 + \sigma^2_{i+\step,*}(\xaug_{\text{ref},i+1})}\right)}_{\leq \lvert \Xref \rvert \frac{1}{2}\log \left(1 +  L_{\sigma_{\step+\lvert\Xref\rvert}}\frac{\lVert \xaug \rVert_2}{\sigma^2_{\text{on}}} \right)}.
	\end{split}
	\end{align}
\end{*proofofthm1}

In particular, \Cref{theorem:MainResult} implies that the mutual information with respect to $\xref^*$ can be approximated arbitrarily accurately by reducing the difference $\xref^*-\xaug_{\step}$. In many control applications, this can be achieved in spite of the model error, e.g., by increasing control gains \cite{Capone2019BacksteppingFP,umlauft2020feedback}. Hence, \Cref{theorem:MainResult} can be potentially employed to guarantee a gradual improvement in model accuracy, or even to derive the worst-case number of iterations required to learn the system.
%


\section{Numerical Experiments}
\label{sect:Examples}

In this section, we apply the proposed algorithm to four different dynamical systems. We begin with a toy example, with which we can easily illustrate the explored portions of the state space. Afterwards, we apply the proposed approach to a pendulum, a cart-pole, and a synthetic model that generalizes the mountain car problem. The exploration is repeated~$50$ times for each system using different starting points~$\bm{x}_0$ sampled from a normal distribution. To quantify the performance of each approach, we compute the root mean square model error (RMSE) on $500$ points sampled from a uniform distribution on the region of interest $\StatSpAugRef$.

We employ a squared-exponential kernel in all examples, and train the hyperparemters online using gradient-based log likelihood maximization \citep{Rasmussen2006}. We employ an MPC horizon of $\NMPC=10$, and choose weight matrix for the MPC optimization step as
\begin{align*}
 \bm{Q} = \sum \limits_{d=1}^{\dimx} \sigma_{g_d} \text{diag}(l^{-2}_{1,g_d}, \ldots, l^{-2}_{\dimx,g_d}) \quad \forall i=1,\ldots,N,
\end{align*}
where $\sigma_{g_d}$ denotes the standard deviation of the GP kernel corresponding to the $d$-th dimension, and $l^{-2}_{1,g_d}, \ldots, l^{-2}_{\dimx,g_d}$ denote the corresponding lengthscales.
In order to ease the computational burden, we apply the first $7$ inputs computed by the LocAL algorithm before computing a new solution. 

We additionally explore each system using a greedy entropy-based cost function, as suggested in \cite{koller2018learning} and 
\cite{schreiter2015safe}, and compare the results. In all three cases, the LocAL algorithm yields a better model in the regions of interest than the entropy-based algorithm.

\subsection{Toy Problem}
\label{subsect:ToyProblem}

Consider the continuous-time nonlinear dynamical system
\begin{align}
\label{eq:ToyProblem}
	\dot{x} = 10(\sin(x) + \arctan(x)+ u),
\end{align}
with state space $\mathcal{X} = \mathbb{R}$ and input space $\mathcal{U} = [-5,5]$. We are interested in obtaining an accurate dynamical model within the region~$\StatSpAugRef = \{ [x \ u]\transp \in \StatSpAug
\ \vert \ x \in [-\pi,\pi], u\in [-1,1]\} $. To obtain a discrete-time system in the form of \eqref{eq:SystemDynamics}, we discretize \eqref{eq:ToyProblem} with a discretization step of $\Delta t = 0.1$ and set the prior model to ${\fsc(x_{\step},u_{\step}) = x}$. The results are displayed in \Cref{fig:ToyProbResults}.

The LocAL algorithm yields a substantial improvement in model accuracy in every run. This is because the system stays close to the region of interest $\StatSpAugRef$ during the whole simulation. By contrast, the greedy entropy-based method covers a considerably more extensive portion of the state space. This comes at the cost of a poorer model on~$\StatSpAugRef$, as indicated by the respective RMSE.

\begin{figure*}[t]
	\setlength{\fboxsep}{1pt}%
	\setlength{\fboxrule}{0pt}%
	\fbox{\includegraphics[scale=0.168]{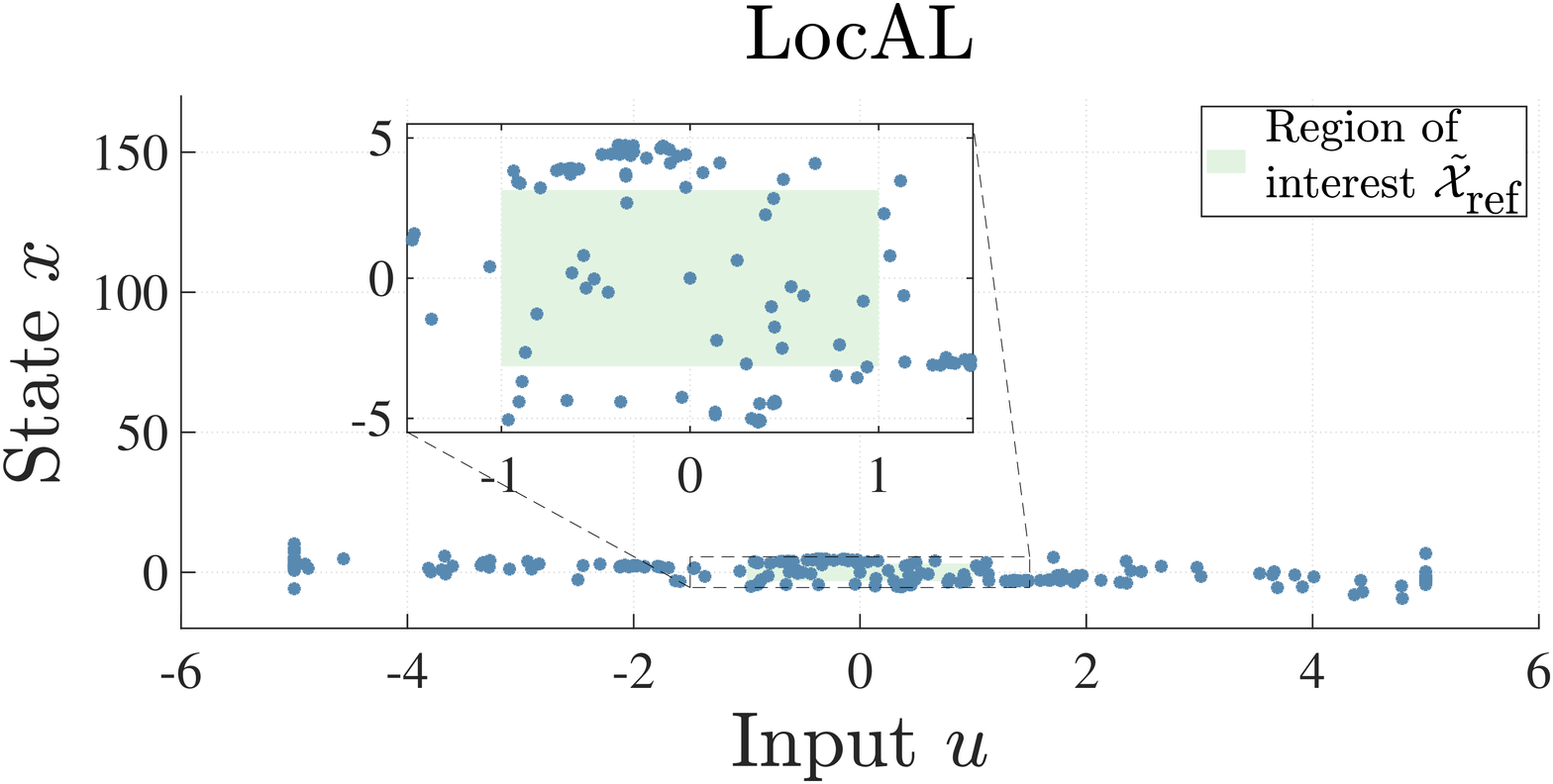}}
	\hfill
	\fbox{\includegraphics[scale=0.168]{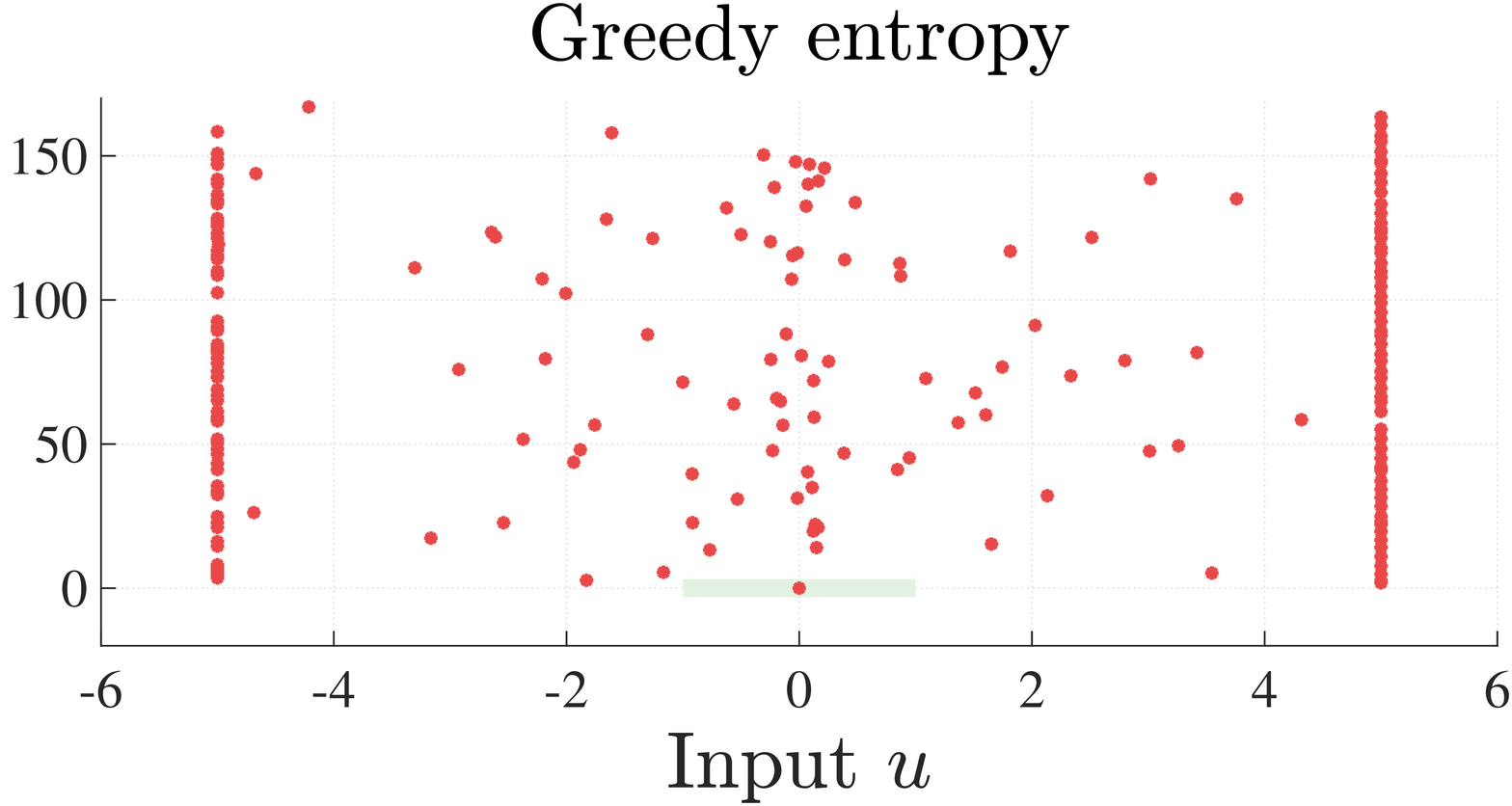}}
		\fbox{\includegraphics[scale=0.168]{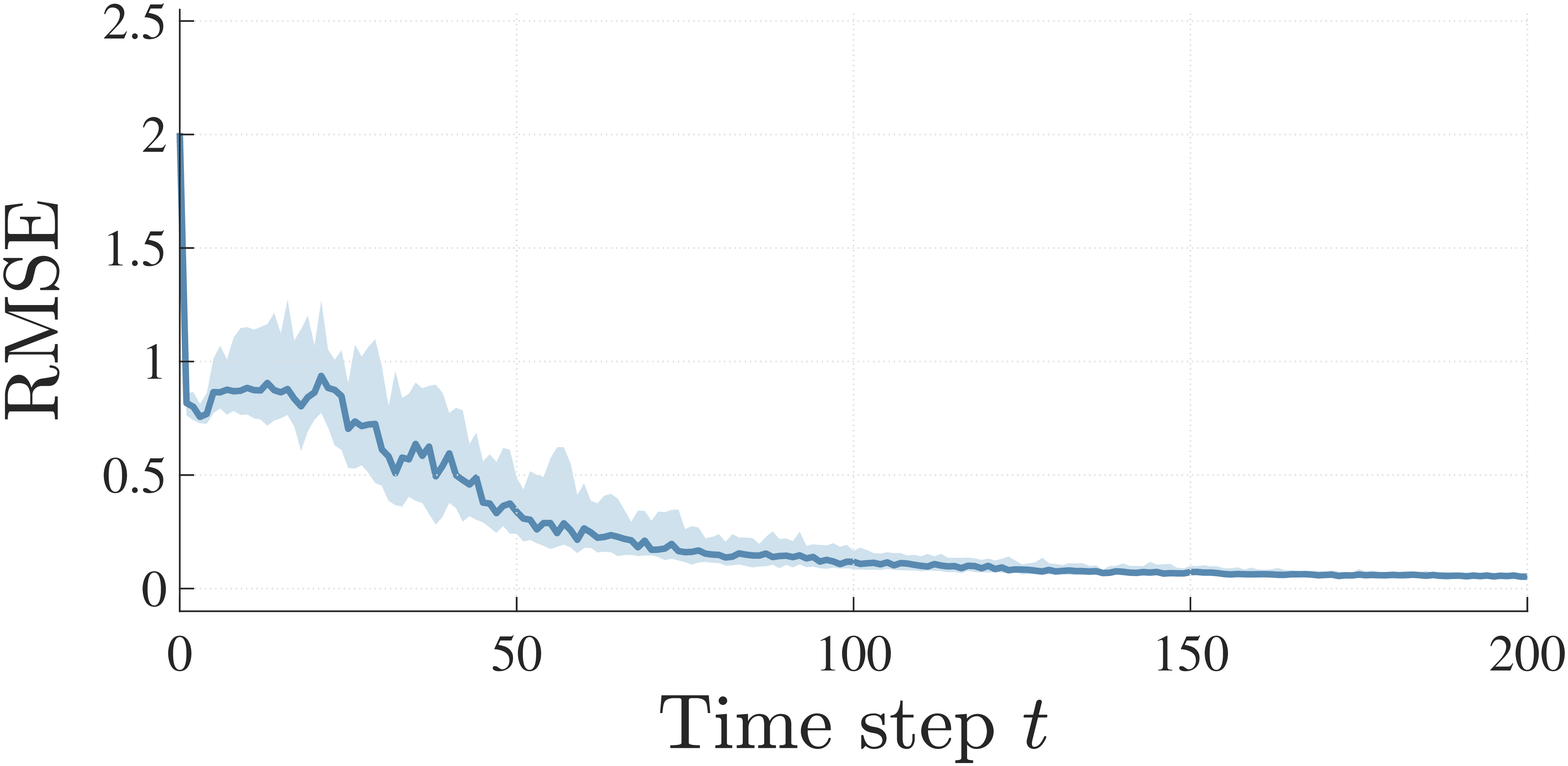}}
		\hfill
	\fbox{\includegraphics[scale=0.168]{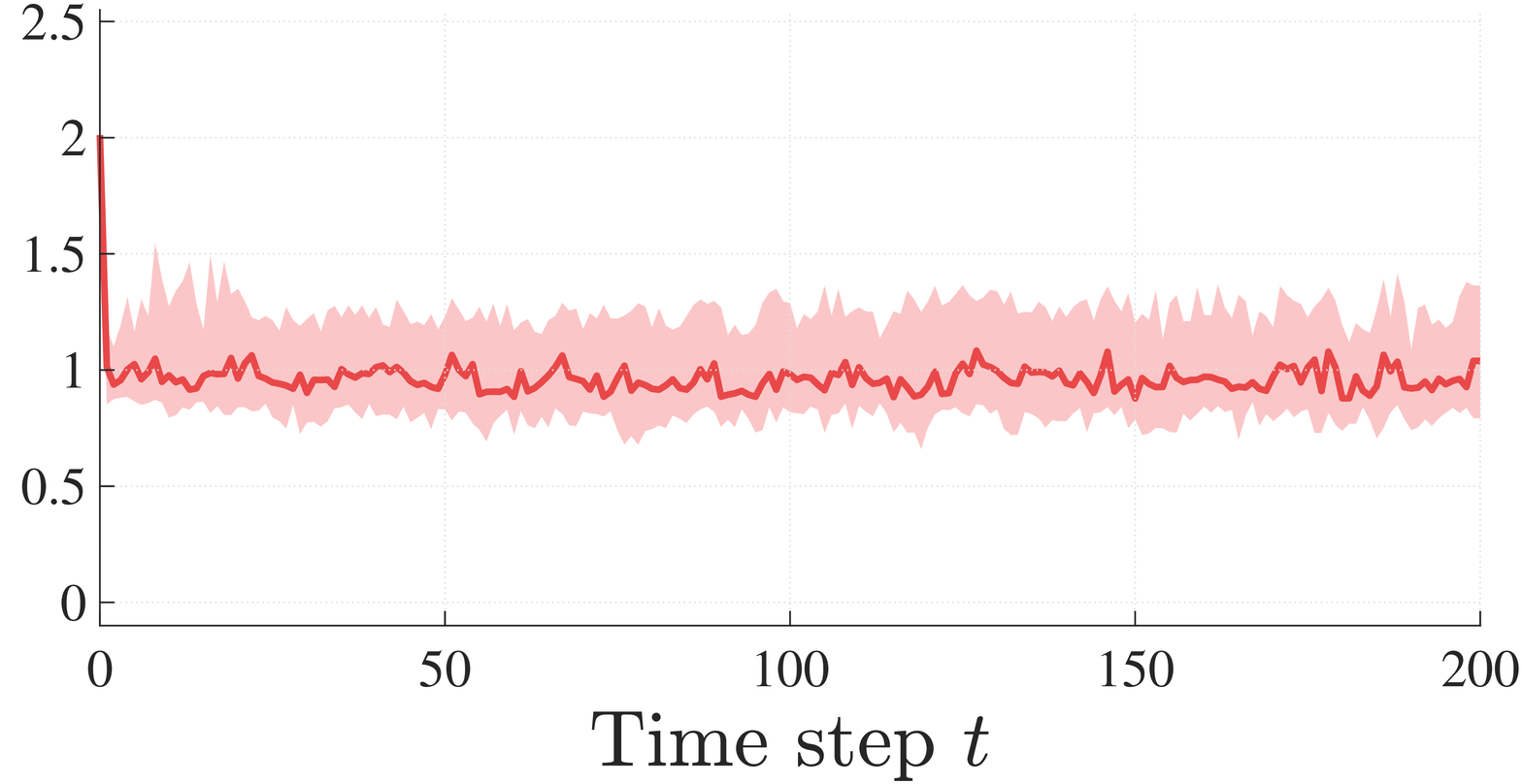}}
	\caption{Toy problem results. Collected data in augmented state space $\X \times \mathcal{U}$ after 200 times steps (top). Median, lower and upper quartile of RMSE on region of interest (bottom). The LocAL algorithm explores the region of interest more thoroughly than the entropy-based approach}
	\label{fig:ToyProbResults}
\end{figure*} 
\subsection{Surface exploration}
We apply the LocAL algorithm to the dynamical system given by
	\begin{align}
	\label{eq:MountainBallDynamics}
	\dot{x}_1 = 3 u_1 + 10 \cos(5 x_1)\cos(5 x_2), \qquad \qquad
	\dot{x}_2 =    3 u_2 + 10 \sin(5 x_1)\sin(5 x_2).
	\end{align}
This setting can be seen as a surface exploration problem, i.e., an agent navigates a surface to learn its curvature. We aim to obtain an accurate model of the dynamics within the region given by $
\StatSpAugRef = \{[\bm{x}\transp \ \bm{u}\transp]\transp \in \StatSpAug \ \vert \ \x \in [-\pi/4,\pi/4]^2, \ \uin \in [-1,1]^2 \}$. To run the LocAL algorithm, we employ a discretization step of $\Delta \step =\SI{0.02}{\s}$ and set the prior model to~${\f(\xaug)= \bm{x} + \Delta \step \uin
}$. The results are shown in \Cref{fig:MountainCarResults}.

The LocAL algorithm manages to significantly improve its model after $400$ time steps, while the entropy-based strategy does not yield any improvement. This is because every variable of the state space is unbounded, i.e., the state space can be explored for a potentially infinite amount of time without ever reaching the region of interest $\StatSpAugRef$. 

\begin{figure*}[t]
	\setlength{\fboxsep}{1pt}%
	\setlength{\fboxrule}{0pt}%
	\fbox{\includegraphics[scale=0.168]{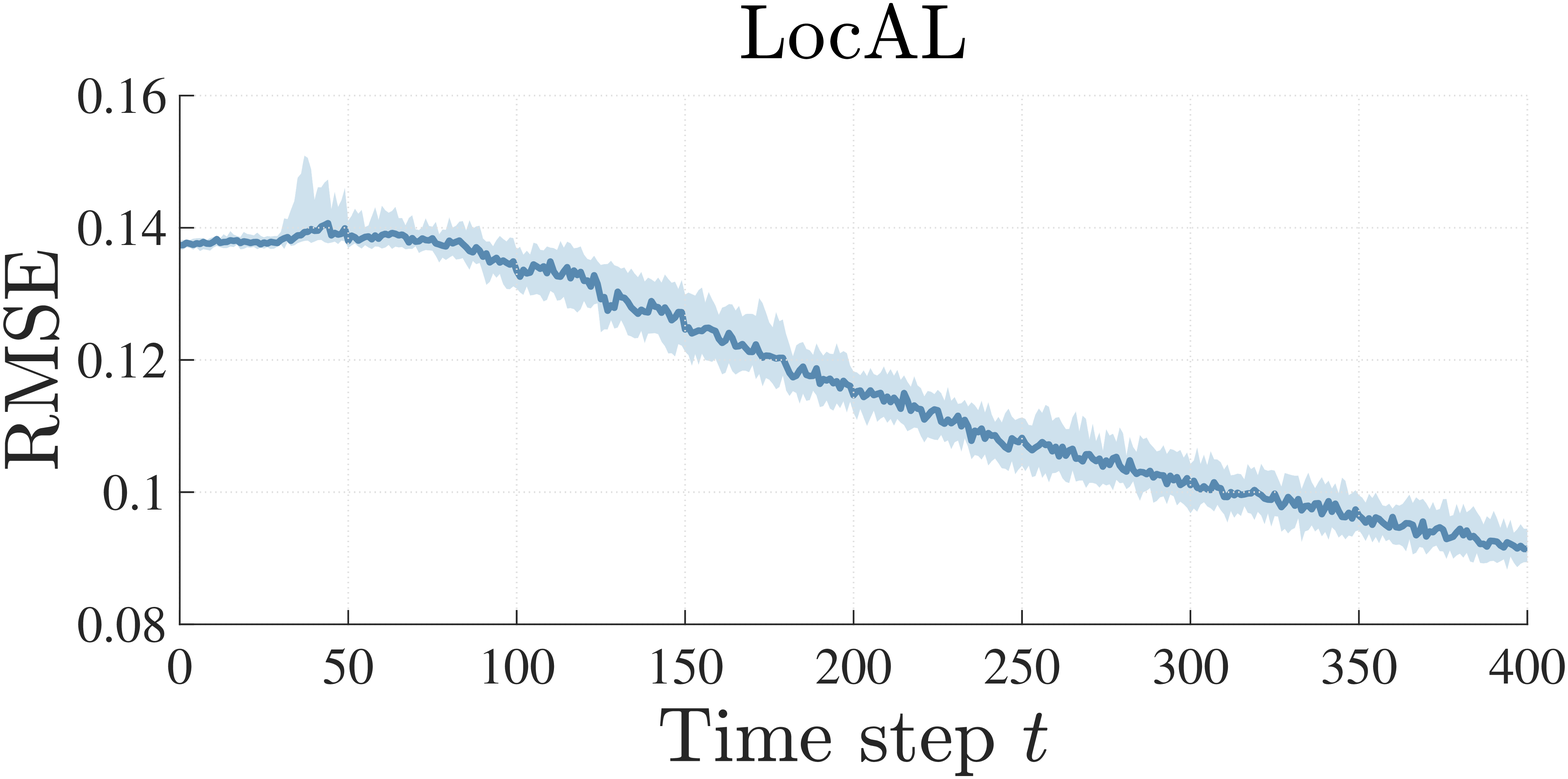}}
	\hfil
	\fbox{\includegraphics[scale=0.168]{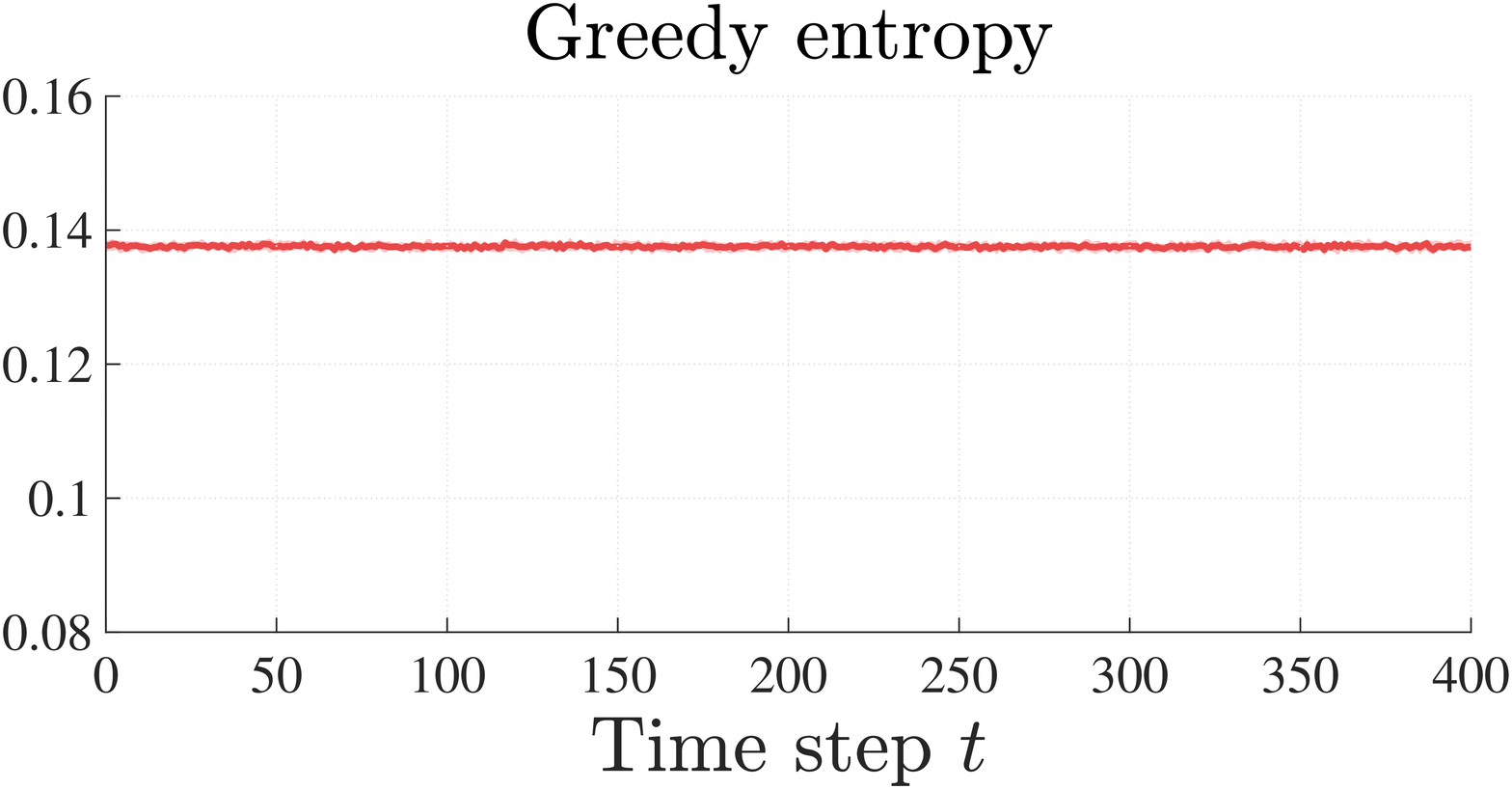}}
	\caption{Surface exploration results. Median, lower and upper quartile of RMSE on region of interest.}
	\label{fig:MountainCarResults}
\end{figure*} 

\subsection{Pendulum}
\label{subsect:Pendulum}

We now consider a two-dimensional pendulum, whose state $\bm{x}=[\vartheta, \dot{\vartheta} ]$ is given by the angle $\vartheta$ and angular velocity $\dot{\vartheta}$. The input torque $u$ is constrained to the interval $\mathcal{U} = [-10,10]$. Our goal is to obtain a suitable model within $
\StatSpAugRef = \{ [\bm{x}\transp  u]\transp \in \StatSpAug \ \vert \  x_1 \in [\pi/2, 3/2\pi], \ x_2  [-5,5], \ u \in [-3,3] \}$.
Obtaining a precise model around this region is particularly useful for the commonly considered task of stabilizing the pendulum around the upward position $\vartheta = \dot{\vartheta}=0$. The results are depicted in \Cref{fig:PendulumResults}.
\begin{figure*}[!t]
	\setlength{\fboxsep}{1pt}%
	\setlength{\fboxrule}{0pt}%
	\fbox{
			\includegraphics[scale=0.168]{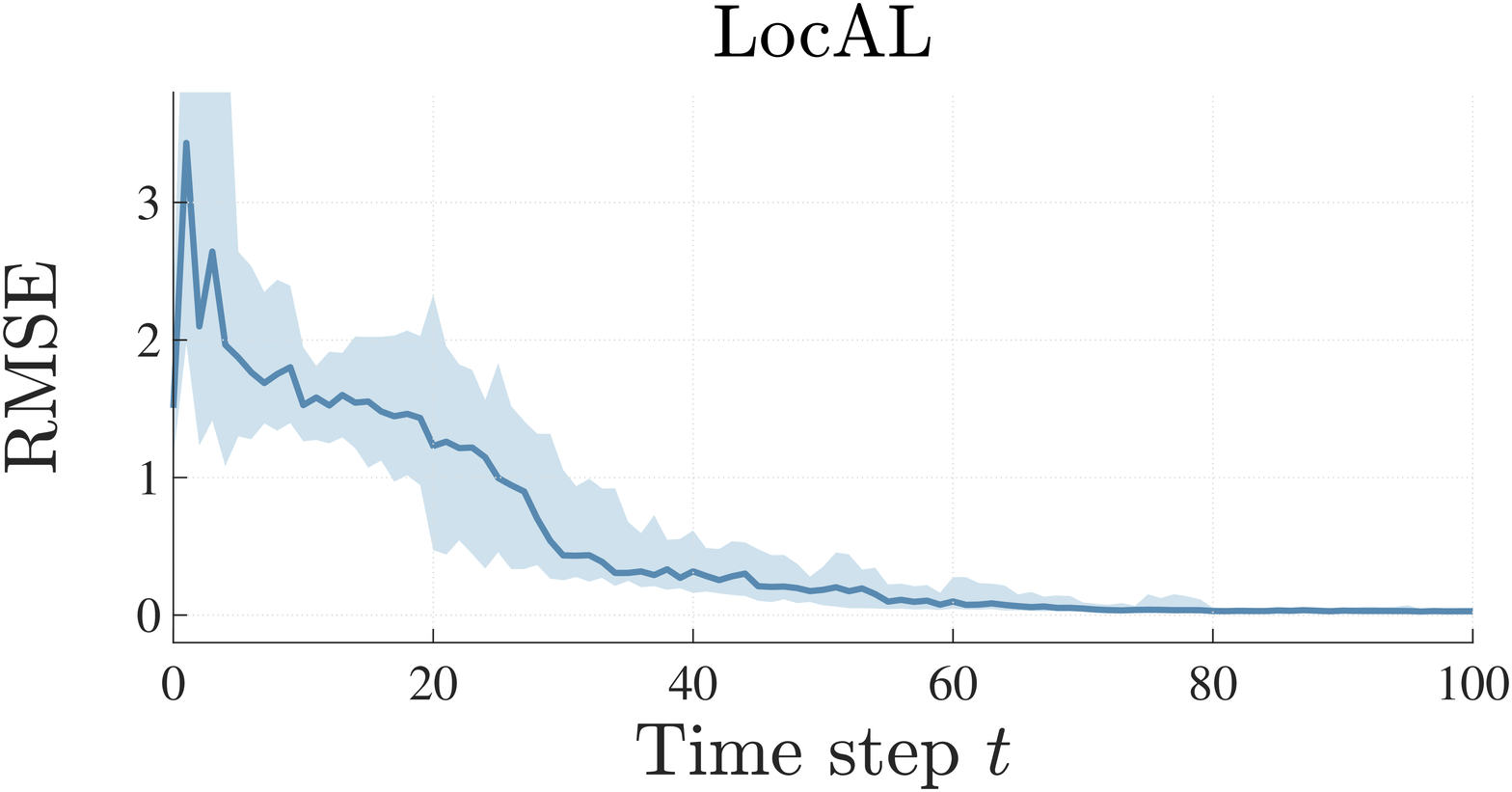}}
		\hfill
		\fbox{
			\includegraphics[scale=0.168]{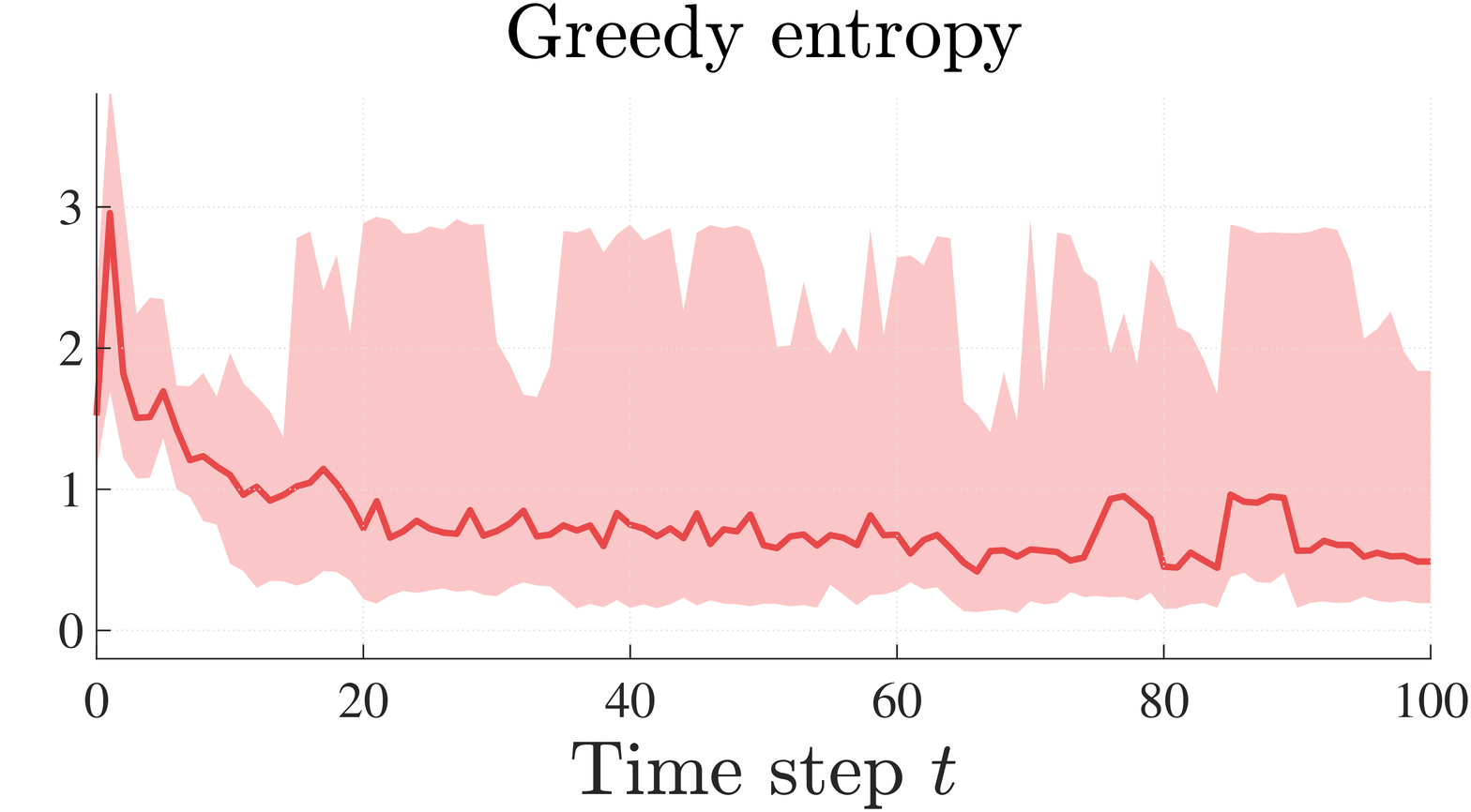}}
	\caption{Pendulum results. Median, lower and upper quartile of RMSE on region of interest.}
	\label{fig:PendulumResults}
\end{figure*} 

The RMSE indicates that the LocAL algorithm yields a similar model improvement every run. The model obtained with the entropy-based strategy, by contrast, exhibits higher errors. 

\subsection{Cart-pole}
\label{subsect:Cart_pole}

We apply the LocAL algorithm to the cart-pole system \citep{6313077}. The state space is given by $\bm{x}=[v, \vartheta, \dot{\vartheta} ]$, where $v$ is the cart velocity, $\vartheta$ is the pendulum angle, and $\dot{\vartheta}$ is its angular velocity. Here we ignore the cart position, as it has no influence on the system dynamics.
The region of interest is given by $
\StatSpAugRef = \left\{[\bm{x}\transp \bm{u}\transp]\transp \in \StatSpAug \  \Big\vert \ x_1 \in [-2,2], \  x_2 \in [-\pi/4,\pi/4], \ u_1,u_2 \in [-5,5] \right\}.$
 Similarly to the pendulum case, obtaining an accurate model on this region is useful to address the balancing task. The discretization step is set to $\Delta t =  \SI{0.05}{\s}$, the prior model is $\bm{f}(\xaug) = \bm{x}$. The results are shown in \Cref{fig:CartPoleResults}.

 \begin{figure*}[!t]
 	\setlength{\fboxsep}{1pt}%
 	\setlength{\fboxrule}{0pt}%
 	\fbox{
 		\includegraphics[scale=0.168]{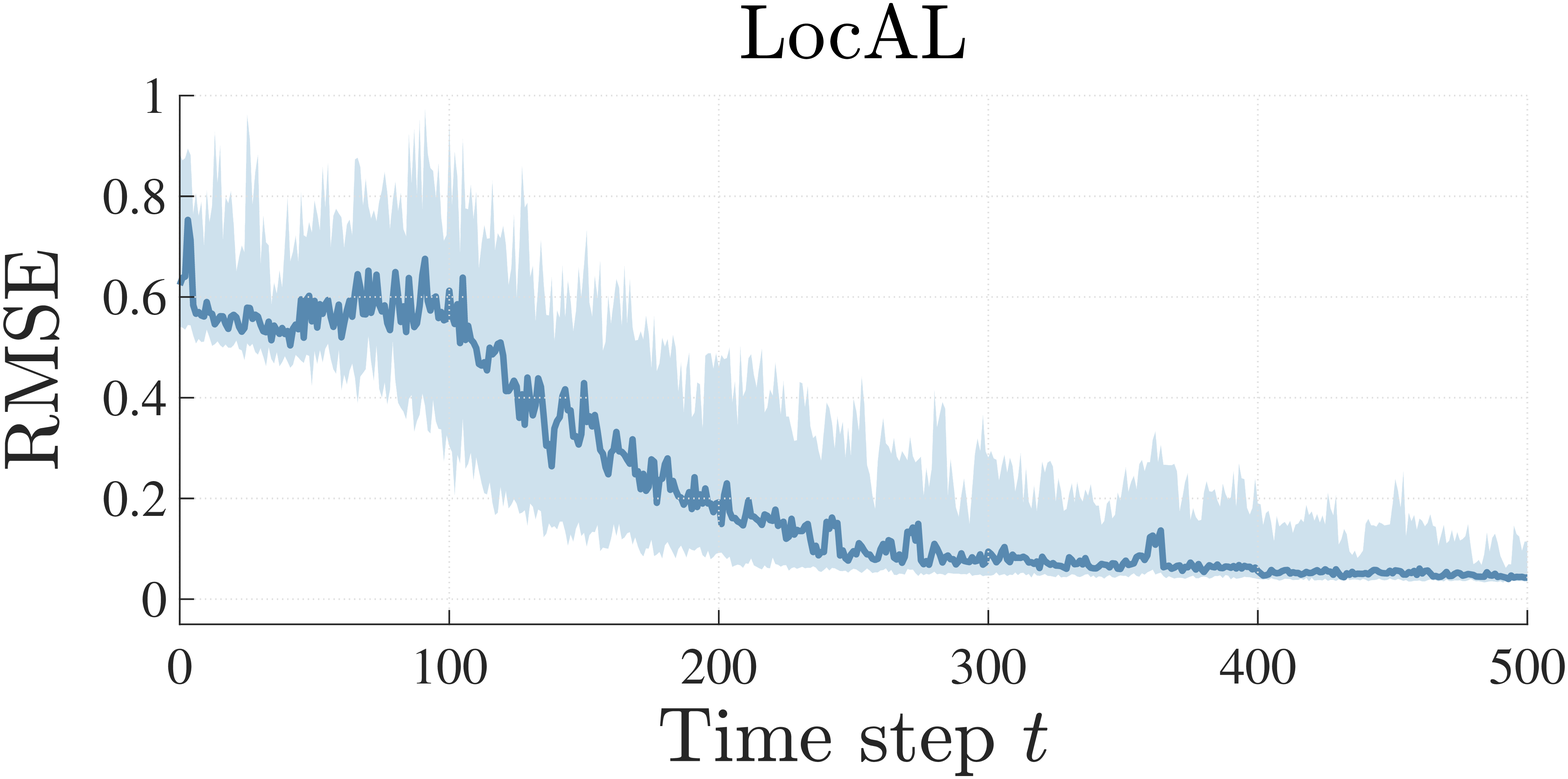}}
 	\hfill
 	\fbox{
 		\includegraphics[scale=0.168]{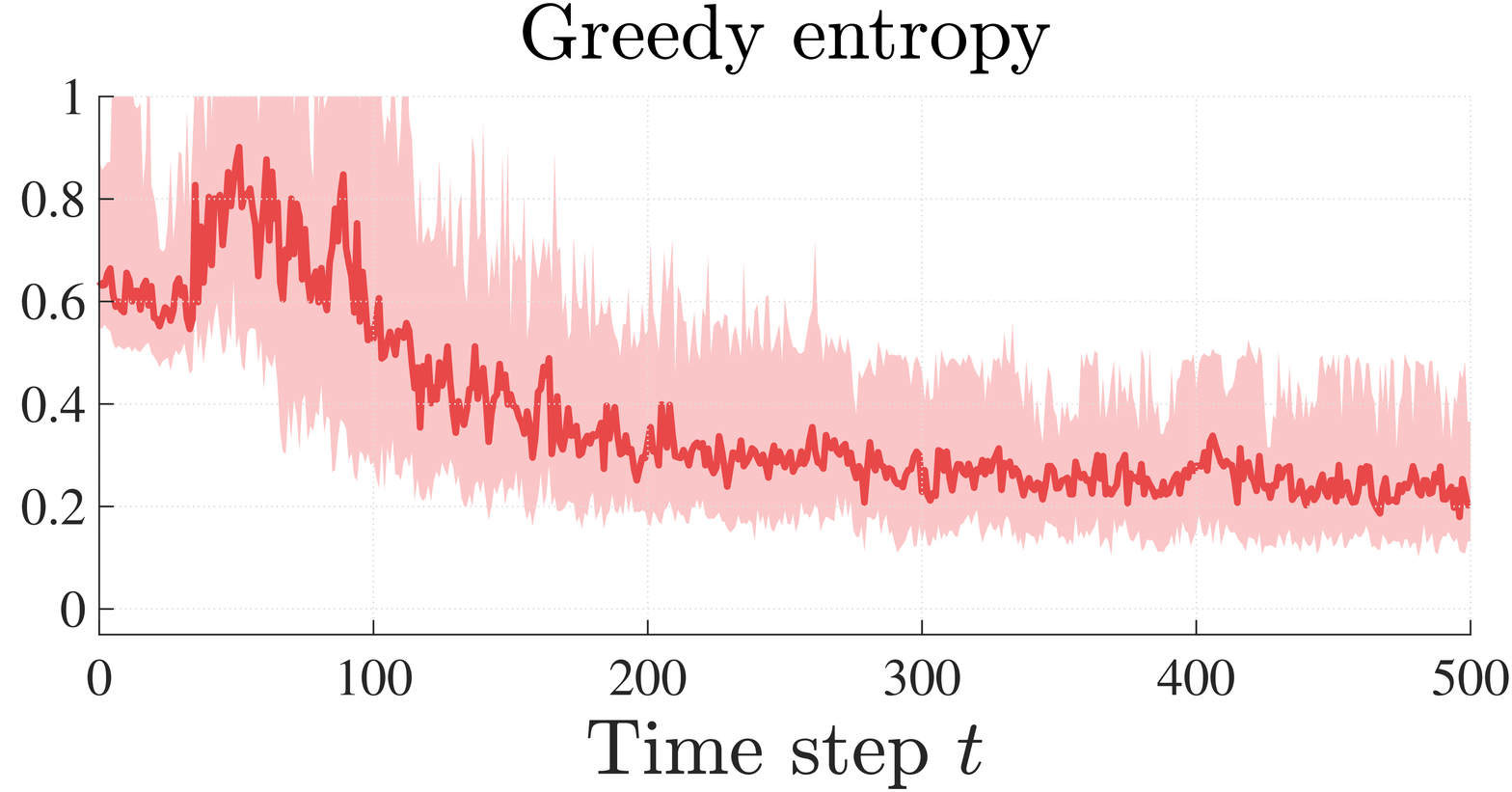}}
 	\caption{Cart-pole results. Median, lower and upper quartile of RMSE on region of interest.}
 	\label{fig:CartPoleResults}
 \end{figure*}

Similarly to the pendulum case, the model obtained with the LocAL algorithm exhibits low standard deviation compared to the one obtained with the entropy-based approach. This is because the region of interest is explored more thoroughly with our approach.

\section{Conclusion}
\label{sect:Conclusion}
A technique for efficiently exploring bounded subsets of the state-action space of a system has been presented. The proposed technique aims to minimize the mutual information of the system trajectories with respect to a discretization of the region of interest. It employs Gaussian processes both to model the unknown system dynamics and to quantify the informativeness of potentially collected data points. In numerical simulations of four different dynamical systems, we have demonstrated that the proposed approach yields a better model after a limited amount of time steps than a greedy entropy-based approach.

\section*{Acknowledgements}
This work was supported by the European Research Council (ERC)
Consolidator Grant ”Safe data-driven control for human-centric systems 
(CO-MAN)” under grant agreement number 864686 .

\bibliography{../Literature/AllPhDReferences}

\end{document}